\documentclass[runningheads]{llncs}
\usepackage{graphicx}
\graphicspath{{images/} }
\makeatletter
\def\temp{dvips.def}
\ifx\Gin@driver\temp
\def\Ginclude@graphics#1{\def\temp{#1}---image \expandafter\strip@prefix\meaning\temp---}
\fi
\makeatother

\usepackage{epsf}
\usepackage{epsfig}

\usepackage{amsmath}
\usepackage{amssymb}
\usepackage{color}
\usepackage{cite}

\newcommand{\ignore}[1]{}

\newcommand{\E}{{\bf E}}
\newcommand{\Var}{{\bf Var}}
\newcommand{\M}{{\cal M}}

\newcommand{\RID}{{ {\rm \scriptscriptstyle RID}}}
\newcommand{\RrSD}{{ {\rm \scriptscriptstyle RrSD}}}
\newcommand{\RsSD}{{ {\rm \scriptscriptstyle RsSD}}}
\newcommand{\TRNS}{{ {\rm \scriptscriptstyle UTDq}}}
\newcommand{\UTDq}{{ {\rm \scriptscriptstyle UTDq}}}

\ignore{

\newtheorem{theorem}{Theorem}

\newtheorem{lemma}[theorem]{Lemma}

}

\begin{document}
\title{Bounds for the Number of Tests\\ in Non-Adaptive Randomized Algorithms\\ for Group Testing}

\titlerunning{Tests Bounds for Non-Adaptive Randomized Group Testing}

\author{Nader H. Bshouty\inst{1} \and
George Haddad\inst{2} \and
Catherine A. Haddad-Zaknoon\inst{1}}

\authorrunning{N. H.  Bshouty et al.}
\institute{Technion, Haifa, Israel \\
\email{\{bshouty,catherine\}@cs.technion.ac.il}\\
 \and
  The Orthodox Arab College, Grade 11, Haifa, Israel\\
\email{haddadgeorge9@gmail.com}
}

\maketitle              % typeset the header of the contribution

\begin{abstract}
We study the group testing problem with non-adaptive randomized algorithms. Several models have been discussed in the literature to determine how to randomly choose the tests. For a model ${\cal M}$, let $m_{\cal M}(n,d)$ be the minimum number of tests required to detect at most $d$ defectives within $n$ items, with success probability at least $1-\delta$, for some constant $\delta$. In this paper, we study the measures
$$c_{\cal M}(d)=\lim_{n\to \infty} \frac{m_{\cal M}(n,d)}{\ln n}
\mbox{\ and \ } c_{\cal M}=\lim_{d\to \infty} \frac{c_{\cal M}(d)}{d}.$$

In the literature, the analyses of such models only give upper bounds for $c_{\cal M}(d)$ and $c_{\cal M}$, and for some of them, the bounds are not tight.
We give new analyses that yield tight bounds for $c_{\cal M}(d)$  and $c_{\cal M}$ for all the known models~${\cal M}$.

\end{abstract}

\section{Introduction}\label{Int}
\emph{Group testing} is a strategy to identify $d$ \emph{defective} items from a pile of $n$ elements by testing groups of items rather than testing each one individually. A group test is identified by a subset of items. The test response is \emph{positive} if it includes at least one defective item, and \emph{negative} otherwise. The problem of group testing is the task of identifying all the $d$ items with a minimum number of group tests.

Formally, let  $S=[n]:=\{1,2,\ldots,n\}$ be the set of the $n$ {\it items} and let $I\subseteq S$ be the set of {\it defective items}. Suppose that we know that the number of defective items, $|I|$, is bounded by some integer~$d$. A {\it test} is a set $J \subset S$. The answer to the test is $T(I,J)=1$ if $I\cap J\not=\O$ and $0$ otherwise. The problem is to find the defective items with a minimum number of tests. \ignore{A test $J\subseteq S$ is identified with an {\it assignment} $a^J\in\{0,1\}^n$, where $a^J_i=1$ if and only if $i\in J$, i.e., $a^J_i=1$ signifies that item $i$ is in test $J$. Then, the answer to the test $a^J$ is $T(I,a^J)=1$ ({\it positive}) if there is $i\in I$ such that $a^J_i=1$ and $0$ ({\it negative}) otherwise.}

Although the group testing scheme was originally introduced as a potential solution for an economical mass blood testing during WWII \cite{D43}, many researchers have suggested applying this approach in a variety of practical problems. Du and Hwang \cite{DH06}, for example, outline a wide range of applications in DNA screening that involve group testing.  On the other hand, Wolf  \cite{W85} presents an applicable group testing generalization to the random access communications problem. For a brief history and other applications, the reader is referred to~\cite{Ang88,Ci13,DH00,DH06,H72,MP04,ND00,KS64,HL02,CM03}.

Generally, the algorithm operates in \emph{stages} or \emph{rounds}. In each round, the tests are defined in advance and are tested in a single parallel step. Tests in some round might depend on the answers of the previous rounds. An algorithm that includes only one stage is called a {\it non-adaptive algorithm}, while a multi-stage algorithm is called an {\it adaptive algorithm}.

Since tests might be time consuming, in most practical applications, performing the tests simultaneously is highly required. Therefore, non-adaptive algorithms are extremely desirable in practice. It is well known, however,  that any non-adaptive {\it deterministic} algorithm must do at least $\Omega(d^2\log n/\log d)$ tests~\cite{DR82,F96,PR11,R94}. This is $O(d/\log d)$ times more than the number of tests of the folklore non-adaptive randomized algorithm that does only $O(d\log n)$ tests. Due to their reduced number of tests, randomized non-adaptive algorithms for group testing have drawn the attention of researchers for the past few decades, and many algorithms have been proposed~\cite{BBKT95,BKB95,DH06,ER63,H00,HL01}.

The set of tests in any non-adaptive deterministic (resp. randomized) algorithm can be identified with a binary (resp. random) $m\times n$  test matrix $M$ (also called pool design). Each row in $M$ corresponds to an assignment $a=(a_1, \cdots, a_n)\in \{0,1\}^n$ where $a_i=1$ if and only if $i\in J$ or equivalently, the $i$th item participates in the test defined by the subset $J$. For random algorithms, the following models are studied in the literature for constructing an $m\times n$ random test matrix $M$.

\begin{enumerate}
\item {\it Random incidence design} (RID algorithms). The entries in $M$ are chosen randomly and independently to be $0$ with probability $p$ and $1$ with probability $1-p$.
\item {\it Random $r$-size design} (RrSD algorithms). The rows in $M$ are chosen randomly and independently from the set of all vectors $\{0,1\}^n$ of weight $r$.
\item {\it Random $s$-set design} (RsSD algorithms). The columns in $M$ are chosen randomly and independently from the set of all vectors $\{0,1\}^m$ of weight $s$.
\item {\it Uniform Transversal Design  with alphabet of size $q$} (UTDq algorithms). The entries of a $q$-ary matrix $M'$ are chosen uniformly and independently at random to be any symbol of the alphabet $\Sigma=\{1,\cdots, q\}$. Then, the resulted matrix is turned into binary matrix $M$. Transforming $q$-ary matrix $M'$ to binary matrix $M$ goes as follows: each row  $r$ in $M'$ is translated to $q$ binary rows in $M$. For each entry value $\sigma \in \Sigma$, replace each entry that is equal to $\sigma$ by $1$ and the others convert to $0$. Therefore, if the matrix $M'$ is of the dimensions $m'\times n$, then $M$ is an $m\times n$ binary matrix where $m = qm'$.
\end{enumerate}
One advantage of  RID and RrSD algorithms over RsSD and UTDq algorithms is that, in parallel machines, the tests can be generated by different processors (or laboratories) without any communication between them. In those models all the machines use the same distribution, draw a sample and perform the test. Those algorithms are called {\it strong non-adaptive} in the sense that the rows of the matrix $M$ can also be non-adaptively generated in one parallel step.

For a model ${\cal M}$, let $m_{\cal M}(n,d)$ be the minimum number of tests of $n$ items with at most $d$ defective items that is required in order to ensure success (finding the defective items) with probability at least $1-\delta$, for some constant $\delta$. In this paper, we study the constant  $ c_{\cal M}$ for each model ${\cal M}$, where $ c_{\cal M}$ is defined as follows:
$$c_{\cal M}(d)=\lim_{n\to \infty} \frac{m_{\cal M}(n,d)}{\ln n}
\mbox{\ and \ } c_{\cal M}=\lim_{d\to \infty} \frac{c_{\cal M}(d)}{d}.$$

To the best of our knowledge, there has been little discussion about any non-trivial lower bound on the number of tests required in a non-adaptive randomized algorithm for group testing~\cite{DA12}. Moreover, the analyses of the previous models known in the literature give only upper bounds for $c_{\cal M}(d)$ and $c_{\cal M}$, and some of these bounds are not tight. For some models, the used techniques do not even lead to an upper bound, and other relaxed measures are examined, such as the expected number of non-defective items that are eliminated after each test.

The objective of this paper is to establish lower and upper bounds on the number of tests required for a non-adaptive randomized algorithm to identify $d$ defectives among $n$ items with a success probability at least $1-\delta$. We develop new techniques that give tight bounds for $c_{\cal M}(d)$  and $c_{\cal M}$ over the models: ${\cal M}=$RID, RrSD, RsSD and UTDq.

\subsection{Old and New Results}

Let $M$ be an $m\times n$ test matrix. Let $I\subseteq S=[n]$, $|I| \leq d$ be the set of defective items. Let $T(I,M)$ denote the vector of answers to the tests (rows of $M$), that is, $T(I,M):=\vee_{i\in I} M^{(i)}$ where $\vee$ is bit-wise ``or'' and $M^{(i)}$ is the $i$th column of $M$. A matrix $M$ is called $(n,I)$-{\it separable} if for every $J\subseteq [n]$, $|J|\le d$ and $J\not=I$, we have $T(J,M)\not=T(I,M)$. That is, the only set $J$ of up to $d$ items that is consistent with the answers of the tests $\vee_{i\in I} M^{(i)}$ is $I$. 

While the separability property is obviously sufficient to guarantee identifying the defectives successfully, unfortunately, the analysis of such property seems to be very involved~\cite{BDKS16}.  Therefore, a more relaxed property is required. A matrix $M$ is called  $(n,I)$-{\it disjunct} with respect to some subset $I\subset [n], |I|\leq d$, if for each $i\not\in I$, there is a test that contains it but does not contain any of the defective items. Formally, for any  $i\not\in I$, there is a row $t\in[m]$ such that $M_{t,i} = 1$ and for all $j\in I$, $M_{t,j}=0$. Since no defective item participates in such test, the response of the oracle on it will be negative ($0$) and hence, is a witness for the fact that the item $i$ is not defective. 

If the test matrix $M$ is $(n,I)$-disjunct, the decoding algorithm reveals the defective items according to the following procedure. It starts with a set  $X=S$. After making all the tests defined by $M$, for every negative answer of a row $a$ in $M$, it removes from $X$ all the items $i$ where $a_i=1$. Since $M$ is $(n,I)$-disjunct, all the non-defective items are guaranteed to have a test that eliminates them from $X$. Therefore, $X$ will eventually contain the defective items only. This can be done in linear time in the size of $M$. 

The folklore non-adaptive randomized algorithm randomly chooses $M$ such that, for any set of at most $d$ defective items $I$, with probability at least $1-\delta$, $M$ is  $(n,I)$-disjunct, and then applies the previous algorithm to identify the defective items. This is why the property of {\it disjunction} is well studied in the literature~\cite{BBKT95,BKB95,DH06,H00,HL01,HL03,HL04}. It is well known (and very easy to see) that if $M$ is $(n,I)$-disjunct then $M$ is $(n,I)$-separable.

Let $\delta$ be some constant. For a model ${\cal M}$, let $m^D_{\cal M}(n,d)$ be the minimum number of tests that is required in order to ensure that for any set of at most $d$ defective items $I$, with probability at least $1-\delta$, the test matrix is $(n,I)$-disjunct. We  define,

$$c^D_{\cal M}(d)=\lim_{n\to \infty} \frac{m^D_{\cal M}(n,d)}{\ln n}
\mbox{\ and \ } c^D_{\cal M}=\lim_{d\to \infty} \frac{c^D_{\cal M}(d)}{d}.$$
Since an $(n,I)$-disjunct matrix is $(n,I)$-separable, for every model $\M$, we have
\begin{eqnarray}\label{UU}
c_\M(d)\le c^D_\M(d).
\end{eqnarray}

Consider $m_{\cal M}(n,d+1)$ random tests in the model ${\cal M}$ and their corresponding matrix $M$. Let $I$ be any set of size $d$. Then, with probability at least $1-\delta$, $M$ is $(n,I)$-separable and therefore for any $j\not\in I$ we have $\vee_{i\in I\cup\{j\}} M^{(i)}\not =\vee_{i\in I} M^{(i)}$. Notice that $\vee_{i\in I\cup\{j\}} M^{(i)}\not =\vee_{i\in I} M^{(i)}$ implies that there is a row $a$ of $M$ that satisfies $a_i=0$ for all $i\in I$ and $a_j=1$. Therefore, with probability at least $1-\delta$, $M$ is $(n,I)$-disjunct. Thus, $m^D_{\cal M}(n,d)\le m_{\cal M}(n,d+1)$ and
\begin{eqnarray}\label{LL}
c_\M(d+1)\ge c^D_\M(d).
\end{eqnarray}
Since $c_\M(d), c^D_\M(d)=O(d)$ and from (\ref{UU}) and (\ref{LL}) it follows that $c_\M= c^D_\M$. The best lower bound for $c_\M$ is
$$c_\M \ge 1/\ln 2 \mbox{\ \ \  and\ \ \ } c_\M(d), c^D_\M(d)\ge d/\ln 2.$$
This bound  follows from the trivial information-theoretic lower bound $d\log n=(d/\ln 2)\ln n$.

Seb\"o, \cite{S85}, studies the RID model for the simple case when the number of defective items is {\it exactly}~$d$. He shows that the best probability for the random test matrix (i.e., that gives a minimum number of tests) is $p=1-1/2^{1/d}$ and, in this case, the number of tests meets the information-theoretic lower bound.

The general case is studied in~\cite{BBKT95,BDKS16,BKB95,DH06,ER63,H00,HL01}. The technique used in most of the studies relies on obtaining the probability that maximizes the expected number of items eliminated in one test. In \cite{BBKT95}, it is shown that for the RID model this probability is $p=1-1/(d+1)$. Moreover, it is shown that  using this probability $c_{\RID}^D(d)\le ed+{(e-1)}/{2}+O\left({1}/{d}\right)$ and therefore
$c_{\RID}=c_{\RID}^D\le e.$~\cite{BBKT95,BDKS16}.
\ignore{
\begin{eqnarray}
c_{\RID}^D(d)\le \gamma_d&:=&\frac{1}{\ln \left(1-\frac{1}{d}\left(1-\frac{1}{d+1}\right)^{d+1}\right)^{-1}}\nonumber\\
&=&ed+\frac{e-1}{2}+O\left(\frac{1}{d}\right).
\end{eqnarray}
and therefore
$$c_{\RID}=c_{\RID}^D\le e.$$}
In their work, Bshouty et. al. \cite{BDKS16} study the separability property and show that
$c_{\RID}(d)\le ed-{(e+1)}/{2}+O\left(1/d \right).$

In this paper we give lower and upper bounds on the number of tests required by any non-adaptive randomized group testing algorithm when tests are chosen according to the RID, RrSD, RsSD and UTDq models. For random designs selected according to the RID model, we show that $$c_{\RID}^D(d)=ed, \mbox{\ \ \ \ and therefore, \ \ \ \ }c_\RID=e=2.718.$$ The optimal probability that derives this result is $p=e^{-1/d}$. This, in particular, shows that finding the probability that maximizes the expected number of items that are eliminated in one test does not necessarily give the probability that minimizes the total number of tests. Considering the RrSD model, we prove that $$c^D_{\RrSD}(d)=c^D_\RID(d)=ed, \mbox{\ \ \ \ and\ \ \ \ }c_\RrSD=c_\RID=e=2.718.$$  Moreover, for the RsSD model, we show that $$c^D_\RsSD(d)= \frac{1}{(\ln 2)\max_{0<\alpha\le 1}\left(H(\alpha)-\beta H\left(\frac{\alpha}{\beta}\right)\right)}   \mbox{\ \ \ \ and\ \ \ \ } c_\RsSD=\frac{1}{(\ln 2)^2}=2.081$$
where $\beta=1-(1-\alpha)^d$. Regarding the UTDq model, we prove that
$$c^D_\UTDq(d)= \min_q\frac{q}{-\ln P_{q,d}}   \mbox{\ \ \ \ and\ \ \ \ } c_\UTDq=\frac{1}{(\ln 2)^2},$$
where $$P_{q,d}=\left(\prod_{i=1}^d\left(\frac{i}{q}\right)^{R_{q,d,i}}\right)^{1/q^d},$$ and $R_{q,d,i}$ is the number of strings in $[q]^d$ that contains exactly $i$ symbols.

In addition, for small $d$, Table~\ref{TBL11} outlines the values of $c_{{\cal M}}^D(d)/d$ across the above four models.

\begin{table}\centering

  \begin{tabular}{|c || l | l|l|l|  }
    \hline
    $d$ &RID& RrSD&RsSD& UTDq \\ \hline \hline
    2 &2.718& 2.718& 1.95 & 2.417 \\ \hline
    3 &2.718& 2.718& 1.96& 2.31\\ \hline
    4 &2.718& 2.718& 1.992& 2.225 \\ \hline
    5 &2.718& 2.718& 2.01 & 2.221\\ \hline
    6 &2.718& 2.718& 2.02 & 2.198\\ \hline
    7 &2.718& 2.718& 2.03  &2.182\\ \hline
   8 & 2.718& 2.718&2.04  & 2.17\\ \hline
    9 &2.718& 2.718& 2.044 & 2.16\\ \hline
    10&2.718& 2.718& 2.05 &2.152\\ \hline
$\to\infty$&$e=$& $e=$&$\frac{1}{\ln^22}=$&$\frac{1}{\ln^22}=$ \\
&2.718& 2.718&2.081&2.081 \\ \hline
  \end{tabular}

\caption{ Leading constant of $d\ln(n)$ for small $d$ and $n\to \infty$ for RID, RrSD, RsSD, and  UTDq models.}
\label{TBL11}
\end{table}

\section{The RID Model}
In this section, we study the Random Incidence Design (RID algorithms). We recall that in this model, the entries in $M$ are chosen randomly and independently to be $0$ with probability $p$, and $1$ with probability $1-p$.
We prove:
\begin{theorem} We have
$$c^D_\RID(d)= ed \mbox{\ \ \ \ and\ \ \ \ } c_\RID= e.$$
\end{theorem}

In Lemma~\ref{UforRID}, by choosing $p=e^{-\frac{1}{d}}$, we develop an upper bound on the number of queries required for the group testing problem when $M$ is designed according to the RID model.  For this choice of $p$, we establish, in  Lemma~\ref{LforRID}, a lower bound that shows that this choice of $p$ gives the minimum number of tests in this model.

As in the introduction, let the set $I$ denote defective items set. We say that a row $i$ in $M$ is a {\it good} row if for every $j\in I$, we have $M_{i,j}=0$.

\begin{lemma}\label{UforRID}
Let M be an $m \times n$ RID  matrix with $p=e^{-\frac{1}{d}}$ where
\begin{equation}
m-\sqrt{2em \ln \frac{2}{\delta }}= e d \ln \frac{2n}{\delta}.
\label{mRID}
\end{equation}

Then, for any $I$ of size at most $d$, with probability at least $1-\delta $, $M$ is an $(n,I)$-disjunct matrix. In particular,
$$
m=ed\ln\frac{2n}{\delta}+\Theta\left(\sqrt{d\ln\frac{n}{\delta}\ln\frac{1}{\delta}}\right), \mbox{\ \ \ \ } c^D_\RID(d)\le ed \mbox{\ \ \ \ and\ \ \ \ } c_\RID\le e.
$$

%$$c^D_\RID(d)\le ed \mbox{\ \ \ \ and\ \ \ \ } c_\RID\le e.$$
\label{RID_UB}
\end{lemma}
\begin{proof}
For any $1\leq i\leq m$, let $X_i$ be a random variable that is equal to $1$ if the row $i$ in $M$ is a good row and $0$ otherwise. The probability that a row $i$ in $M$ is a good row is $(e^{-1/d})^{d }= e^{-1}$. Let $X = X_1 + \cdots + X_m$, be the number of the good rows in $M$. Then, $\E[X_i] = \Pr[X_i=1]=e^{-1}$ and $\mu := \E[X] = e^{-1}m$. Let $m'=ed\ln(2n/\delta)$.  Let $A$ be the event indicating that the number of the good rows in $M$ is less than $m'p^{d} = m'/e$. Let $T\subseteq[m]$  be the set of the good row indexes in $M$. That is, for each row $t\in T$, $M_{t,j}=0$ for all $j\in I$. Let $B$ be the event indicating that there is a column $j \notin I$ in $M$ such that $M_{t,j} = 0$ for all  $t\in T$,  and therefore, $M$ is not $(n,I)$-disjunct. Then we can say that,

\begin{equation}
\Pr[B|\overline{A}] \leq (n-d) p^{m'p^d} \leq n e^{-\frac{m'}{ed}} = \frac{\delta}{2}.
\label{res2}
\end{equation}

Using Chernoff bound, Lemma~\ref{Chernoff}, for $m$ as specified in (\ref{mRID}) and since $\mu = e^{-1}m$ we can conclude,
\begin{eqnarray}
\Pr[A]= \Pr\left[X < \frac{m'}{e}\right] &=&  \Pr \left[X < \left (1-\left (1-\frac{m'}{m}\right)\right)\mu\right]  \leq e^{-\frac{(1-\frac{m'}{m})^2m}{2e}} = \frac{\delta}{2}.  \label{res1}\nonumber \\
%& = & \Pr \left[X < \left (1-\left (1-\frac{m'}{m}\right)\right)\mu\right] \nonumber \\
%&\leq &e^{-\frac{(1-\frac{m'}{m})^2m}{2e}} = \frac{\delta}{2}.  \label{res1}
\end{eqnarray}
Using (\ref{res2}) and (\ref{res1}) we get,
\begin{align*}
\Pr[B] &= \Pr[B|A]\Pr[A]+\Pr[B|\overline A]\Pr[\overline A]\le \Pr[A]+\Pr[B|\overline A] \le \delta.\qed
 \end{align*}
\end{proof}

\begin{lemma}\label{LforRID} Let $M$ be an $m \times n$ RID matrix with probability $0\le p\le 1$. If $$m= ed\ln n -e^{3}\sqrt{3m},$$
then with probability at least $1/3$, $M$ is not $(n,I)$-disjunct.

In particular, to have success probability at least $1/3$, we must have
$$m\ge ed\ln n-\Theta(\sqrt{d \ln n}).$$ 
Therefore, $$c^D_\RID(d)\ge ed \mbox{\ \ \ \ and\ \ \ \ } c_\RID\ge e.$$
\end{lemma}
\begin{proof} Consider $c=d\ln(1/p)$. Then $p=e^{-c/d}$.
Let $m'=ed\ln n$ and $w=e^{3}\sqrt{3m}$. Then $m=m'-w$. Let $X$ be a random variable that is equal to the number of good rows in $M$. Then $\mu:= \E[X]=e^{-c}m$. Let $F$ be the event that $M$ is not $(n,I)$-disjunct. The probability that $M$ is $(n,I)$-disjunct is the probability that in every column that corresponds to a non-defective item, not all the entries of the good rows are zero. Therefore, given that $X=x$, it is easy to see that $\Pr[\overline{F}|X=x]=(1-p^x)^{n-|I|}$. We distinguish between two cases: $c\ge 6$ and $c\le 6$.

Case I. $c\le 6$. By Chernoff bound, Lemma~\ref{Chernoff}, and since $ce^{1-c}\le 1$ for every $c\le6$,
\begin{align*}
\Pr[F] &\ge \Pr[F\ \wedge\ X\le e^{-c} m'] =\Pr[F\ |\ X\le e^{-c} m']\cdot \Pr[X\le e^{-c} m'] \\
&\ge \left(1-\left(1-p^{e^{-c} m'} \right)^{n-d}\right) \left(1-\Pr\left[X\ge e^{-c} m'\right]\right) \\
&\ge \left(1-\left(1-n^{-ce^{1-c}} \right)^{n-d} \right) \left(1-\Pr\left[X\ge e^{-c} m\left(1+\frac{w}{m}\right) \right]\right) \\
&\ge  \left(1-\left(1-n^{-ce^{1-c}}\right)^{n-d}\right) \left(1-e^{-\frac{e^{-c} w^2}{3m} } \right) \\
&\ge \left( 1-\left(1-\frac{1}{n}\right)^{n-d} \right) \left(1-e^{-1} \right) =  \left(1-e^{-1+o(1)} \right)  \left(1-e^{-1} \right) > \frac{1}{3},
\end{align*}
where the last inequality is correct due to the fact that $d=o(n)$. \\

Case II. $c\ge 6$. By Markov bound, Lemma~\ref{Markov}, since $m\le ed\ln n$ and $2ce^{1-c}<0.1$, then for $c\ge 6$ and $d=o(n)$ we have,
\begin{align*}
\Pr[F] &\ge \Pr[F\ \wedge\ X\le 2e^{-c} m] =\Pr[F\ |\ X\le 2e^{-c} m]\cdot (1-\Pr[X\ge 2e^{-c} m]) \\
&\ge \left(1-\left(1-p^{2e^{-c} m} \right)^{n-d}\right)\cdot \frac{1}{2} \ge \frac{1}{2}\left(1-\left(1-p^{2e^{-c} ed\ln n} \right)^{n-d} \right)\\
& \ge  \frac{1}{2}\cdot\left(1-\left(1-n^{-2ce^{1-c}}\right)^{n-d}\right)  \ge \frac{1}{2}\cdot\left( 1-\left(1-\frac{1}{n^{0.1}}\right)^{n-d} \right) > \frac{1}{3}.\qed
\end{align*}
\end{proof}

\section{The RrSD model}
In this section we study the Random $r$-Size Design (RrSD algorithms). As defined previously, in this model, the rows in $M$ are chosen randomly and independently from the set of all vectors $\{0,1\}^n$ of weight $r$.
\begin{theorem} We have
$$c^D_\RrSD(d)= ed \mbox{\ \ \ \ and\ \ \ \ } c_\RrSD= e.$$
\end{theorem}

In Lemma \ref{UforRrSD}, we give an upper bound using $r=(1-e^{-\frac{1}{d}})n$. In Lemma~\ref{LforRrSD}, we establish a lower bound that shows that this choice of $r$ gives the minimum number of tests for this model. The proof of both lemmas is very similar to the upper and lower bound proofs of the RID model.

\begin{lemma}\label{UforRrSD}\label{RrSDT}
Let M be an $m \times n$ RrSD  matrix with $r=(1-p)(n-d+1)$, where $p=e^{-\frac{1}{d}}$ and  $$m-\sqrt{2em \ln \frac{2}{\delta }}= e d \ln \frac{2n}{\delta}.$$

Then, for any $I$ of size at most $d$, with probability at least $1-\delta $, $M$ is an $(n,I)$-disjunct matrix. In particular,
$$m=ed\ln\frac{2n}{\delta}+\Theta\left(\sqrt{d\ln\frac{n}{\delta}\ln\frac{1}{\delta}}\right),\mbox{\ \ \ \ } c^D_\RrSD(d)\le ed \mbox{\ \ \ \ and\ \ \ \ } c_\RrSD\le e.$$
%$$c^D_\RrSD(d)\le ed \mbox{\ \ \ \ and\ \ \ \ } c_\RrSD\le e.$$
\end{lemma}

\begin{proof}  Let $m'=ed\ln(2n/\delta)$. The probability that a row in $M$ is a good row is equal to
$$\frac{{n-d\choose r}}{{n\choose r}}=\left(1-\frac{r}{n}\right)\left(1-\frac{r}{n-1}\right)\cdots \left(1-\frac{r}{n-d+1}\right)\ge p^d=e^{-1}.$$
For a row $i$, let $X_i$ be a random variable that is equal to $1$ if row $i$ is good and $0$ otherwise. Let $X=X_1+\cdots+X_m$ be the number of good rows in $M$. Then, $\E[X_i]\ge e^{-1}$ and $\mu:=\E[X]\ge e^{-1}m$. Let $A$ be the event indicating that the number of good rows in $M$ is less than $m'p^d=m'/e$. Let $B$ be the event that there is a column in $M$ that corresponds to a non-defective item that has entries that are equal to zero in all the good rows, that is, $M$ is not $(n,I)$-disjunct. Since $(m'/m)\mu\ge (m'/m)(e^{-1}m)=m'/e$, by Chernoff bound, Lemma~\ref{Chernoff}, we can say:
\begin{align*}
\Pr[A]&\le \Pr\left[X <\left(1-\left(1-\frac{m'}{m}\right)\right)\mu\right] \le e^{-\frac{\left(1-\frac{m'}{m}\right)^2m}{2e}} = \frac{\delta}{2}.
\end{align*}

On the other hand, we also have,
\begin{align*}
\Pr\left[B| \overline A\right] &\le n\left(\frac{{n-d-1\choose r}}{{n-d\choose r}}\right)^{m'p^d}=n\left(1-\frac{r}{n-d}\right)^{m'p^d}\\
&\le np^{m'p^d} =ne^{-\frac{m'}{ed}} \le ne^{\ln\frac{\delta}{2n}} =\frac{\delta}{2}.
\end{align*}
This gives,
\begin{align*}
\Pr[M\text{\ is not $(n,I)$-disjunct}] &= \Pr[B] = \Pr[B|A]\Pr[A]+\Pr[B|\overline A]\Pr[\overline A]\\ &\le \Pr[A]+\Pr[B|\overline A] \le \delta.\qed
 \end{align*}
\end{proof}

\begin{lemma}\label{LforRrSD} Let $M$ be an $m \times n$ RrSD matrix where $0\le r\le n$ and $d<n^{1/2}/\ln^3n$. If $$m= ed\ln (n/e) -e^{3}\sqrt{3m},$$
then, with probability at least $1/3$, $M$ is not $(n,I)$-disjunct. In particular, to have success probability at least $1/3$, we must have
$$m\ge ed\ln n-\Theta(d+\sqrt{d \ln n}).$$
And therefore, $$c^D_\RrSD(d)\ge ed \mbox{\ \ \ \ and\ \ \ \ } c_\RrSD\ge e.$$
\end{lemma}

\begin{proof} We may assume without loss of generality that $I=\{1,2,\ldots,d\}$. Let $p=1-r/n$ and $c=d\ln(1/p)$. Then $p=e^{-c/d}$.
Let $m'=ed\ln (n/e)$ and $w=e^{3}\sqrt{3m}$. Then $m=m'-w$. Let $X$ be a random variable that is equal to the number of good rows in $M$. Then, as in the previous proof,
$$\mu:=\E[X]=\left(1-\frac{r}{n}\right)\left(1-\frac{r}{n-1}\right)\cdots \left(1-\frac{r}{n-d+1}\right)\cdot m \le p^dm=e^{-c}m.$$
If $c\ge \ln(3m)$, then $\mu\le 1/3$. By Markov's bound, Lemma~\ref{Markov}, $\Pr[X\ge 1]\le 1/3$. Therefore, with probability at least $2/3$, no good rows exist and the algorithm fails. Hence, we may assume that $c<\ln(3m)$ or equivalently, $p\ge 1/(3m)^{1/d}$.

Let $A$ be the event that $M$ is not $(n,I)$-disjunct. The probability that $M$ is not $(n,I)$-disjunct given that $X=x$ is the probability that there is a column that corresponds to a non-defective item, that all its entries in the $x$ good rows are zero. Let $Y_i$ be a random variable that is equal to $1$ if the column $i$ is zero in all the $x$ good rows, and is equal to $0$ otherwise. Let $Y=Y_{d+1}+\cdots+Y_n$. The probability that a fixed column $i$ is zero in the entries of the good rows is $\E[Y_i]=(1-r/(n-d))^x$. The probability that two fixed columns get zeros in all the $x$ good rows is $\E[Y_iY_j]=(1-r/(n-d))^x(1-r/(n-d-1))^x$. Therefore,
$$\hat\mu:=\E[Y]=(n-d)\left(1-\frac{r}{n-d}\right)^x\le np^x,$$ and 
$$\hat\mu_2:=\E[Y_iY_j]=(1-r/(n-d))^x(1-r/(n-d-1))^x\leq p^{2x}.$$
Since $x\le m\le ed\ln(n/e)$, $p=1-r/n$ and $1/p\le (3m)^{1/d}$,
\begin{eqnarray*}
\hat\mu^2&=&(n-d)^2\left(1-\frac{r}{n-d}\right)^{2x}=n^2\left(1-\frac{d}{n}\right)^2p^{2x}\left(1-\frac{dr}{(n-d)(n-r)}\right)^{2x}\\
&=&
(np^x)^2\left(1-\frac{d}{n}\right)^2\left(1-\frac{d(1-p)}{(n-d)p}\right)^{2x}\\
&\ge& (np^x)^2 \left(1-\frac{2d}{n}-\frac{4xd}{np}\right)\ge (np^x)^2\left(1-\lambda\right),
\end{eqnarray*} where $\lambda=o\left(({d^2\ln^2n})/{n}\right)=o(1).$
By Chebychev inequality, Lemma~\ref{Chebychev}, we can say,
\begin{eqnarray*}\Pr\left[\overline{A}|X=x\right]&\le&   \Pr[Y=0|X=x]\le \Pr[|Y-\hat\mu|\ge \hat\mu|X=x]\\
&\le&  \frac{\hat\mu +n(n-1)\hat\mu_2-\hat\mu^2}{\hat\mu^2}\le \frac{np^x+(np^x)^2-(np^x)^2 \left(1-\lambda\right)}{(np^x)^2 (1-\lambda)}\\
&\le& \frac{1+\lambda np^x }{np^x (1-\lambda)}=\frac{1 }{np^x (1-\lambda)}+o(1).
\end{eqnarray*}

We distinguish between two cases: $c\ge 6$ and $c\le 6$.

Case I. $c\le 6$. Since $ce^{1-c}\le 1$, for $n\ge 3$, $n^{1-ce^{1-c}}e^{ce^{1-c}}\ge e$. Since $d<n/\ln^3n$, by Chernoff bound, Lemma~\ref{Chernoff}, and for large enough $n$, we conclude,
\begin{align*}
\Pr[A] &\ge \Pr[A\ \wedge\ X\le e^{-c} m'] =\Pr[A\ |\ X\le e^{-c} m']\cdot \Pr[X\le e^{-c} m'] \\
&\ge \Pr[A\ |\ X= e^{-c} m']\cdot \Pr[X\le e^{-c} m'] \ \ \ \mbox{by Lemma~\ref{mono}}\\
&= \left(1-\frac{1}{n^{1-ce^{1-c}}e^{ce^{1-c}}(1-\lambda)}-o(1)\right) \left(1-\Pr\left[X\ge e^{-c} m'\right]\right) \\
&\ge  \left(1-\frac{1}{e(1-\lambda)}-o(1)\right) \left(1-\Pr\left[X\ge e^{-c} m'\right]\right) \\
&= \left(1-\frac{1}{e(1-\lambda)}-o(1)\right) \left(1-\Pr\left[X\ge e^{-c} m\left(1+\frac{w}{m}\right) \right]\right) \\
&\ge  \left(1-\frac{1}{e(1-\lambda)}-o(1)\right) \left(1-e^{-\frac{e^{-c} w^2}{3m} } \right)\\
&\ge  \left(1-\frac{1}{e(1-\lambda)}-o(1)\right) \left(1-e^{-1}\right)  > \frac{1}{3}.
\end{align*}

Case II. $c\ge 6$. By Markov bound, Lemma~\ref{Markov}, and since $m\le ed\ln (n/e)$ and $2ce^{1-c}<0.1$ for $c\ge 6$, we have
\begin{align*}
\Pr[A] &\ge \Pr[A\ \wedge\ X\le 2e^{-c} m] =\Pr[A\ |\ X\le 2e^{-c} m]\cdot (1-\Pr[X\ge 2e^{-c} m]) \\
&\ge\Pr[A\ |\ X= 2e^{-c} m]\cdot (1-\Pr[X\ge 2e^{-c} m]) \\
&\ge \left(1-\frac{1}{np^{2e^{-c}m}(1-\lambda)}-o(1)\right)\cdot \frac{1}{2} \\
&\ge \left(1-\frac{1}{n^{0.9}(1-\lambda)}-o(1)\right)\cdot \frac{1}{2}\ge \frac{1}{3}.\qed
\end{align*}
\end{proof}

%$$\mbox{\Huge STOP Here}$$
\section{The RsSD model}
In this section we study the Random $s$-Set Design (RsSD algorithms). We recall that in this model, the columns of $M$ are chosen randomly and independently from the set of all vectors $\{0,1\}^m$ of weight $s$.

We prove,
\begin{theorem}
We have,
$$c^D_\RsSD(d)= \frac{1}{(\ln 2)\max_{0<x\le d}\left(H(\alpha)-\beta H\left(\frac{\alpha}{\beta}\right)\right)}   \mbox{\ \ \ \ and\ \ \ \ } c_\RsSD=\frac{1}{(\ln 2)^2}$$
where $\alpha={x}/{d}$ and $\beta=1-(1-\alpha)^d$.
\end{theorem}
An upper bound for this model is established in Lemma \ref{RSSD_BND} using $s=\ln 2\cdot m/d$.  Moreover, in Lemma~\ref{kkkk}, a lower bound is developed indicating that this choice of $s$ gives the minimum number of tests for this model.

%Now we prove the upper bound
 \begin{lemma} \label{RSSD_BND}
Let $M$ be an $m \times n$ RsSD matrix with $\alpha=\frac{s}{m}=\frac{x}{d}$ and $\beta=1-(1-\alpha)^d$ where $0<x\le d$ is any real number.
Let
$$m'=\frac{\ln n+\ln\frac{3}{\delta}+\frac{1}{2}\ln\left(\frac{\beta(1-\alpha)}{\beta-\alpha}\right)}{(\ln 2)\left(H(\alpha)-\beta H\left(\frac{\alpha}{\beta}\right)\right)},$$
and $ \lambda=\frac{2}{\sqrt{\delta (1-\alpha)^d m}}<1 $. Then, for
$
 m=(1+\lambda)m',
$
with probability at least $1-\delta$, $M$ is an $(n,I)$-disjunct matrix.
In particular
$$c^D_\RsSD(d)\le \frac{1}{(\ln 2)\max_{0<x\le d}\left(H(\alpha)-\beta H\left(\frac{\alpha}{\beta}\right)\right)}   \mbox{\ \ \ \ and\ \ \ \ } c_\RsSD\le \frac{1}{(\ln 2)^2}.$$
\end{lemma}

\begin{proof} Let $X_i$ be a random variable that is equal to $1$ if the $i$th row in $M$ is a good row and $0$ otherwise.
Let $X=\sum_{i=1}^{m} X_i $ be the number of good rows in $M$. Let $a=\sqrt{\frac{2(1-\alpha)^d m}{\delta}}$. Since $\lambda<1$, we have
\begin{align*}
\E[X]-a&=(1-\alpha)^dm-a=(1-\alpha)^d(1+\lambda)m'-\sqrt{\frac{2(1-\alpha)^d m}{\delta}}\\
&=(1-\alpha)^d\left(1+\frac{2}{\sqrt{\delta (1-\alpha)^d m'}}\right)m'-\sqrt{\frac{2(1-\alpha)^d (1+\lambda)m'}{\delta}}\\
&=(1-\alpha)^dm'+2\sqrt{\frac{(1-\alpha)^d m'}{\delta}}-\sqrt{\frac{2(1-\alpha)^d (1+\lambda)m'}{\delta}}\ge(1-\alpha)^dm'.
\end{align*}

Let $A$ be the event that $M$ is not an $(n,I)$-disjunct matrix. By Chebychev inequality, Lemma~\ref{Chebychev}, Lemma~\ref{varv}, Lemma \ref{mono} and Lemma \ref{STRLNG} we have,
\begin{align}
\Pr[A]&=\Pr[A|X\le \E[X]-a]\cdot \Pr[X\le \E[X]-a]+\\ &
\ \ \ \ \ \ \ \ \ \Pr[A|X> \E[X]-a]\cdot\Pr[X> \E[X]-a] \nonumber\\
&\le  \Pr[X\le \E[X]-a]+\Pr[A|X> \E[X]-a] \nonumber \\
&\le  \Pr[|X-\E[X]|\ge a]+\Pr[A|X=(1-\alpha)^dm'] \nonumber \\
&\le  \Pr[|X-\E[X]|\ge a]+n \frac{ {\beta m' \choose \alpha m'} }{{m' \choose \alpha m'}} \label{eq1}  \\
&\le \frac{\Var[X]}{a^2}+n\sqrt{\frac{\beta(1-\alpha)}{\beta-\alpha}} 2^{(\beta H(\frac{\alpha}{\beta})-H(\alpha))m'}  (1+o(1))\nonumber \\
&\le \frac{(1-\alpha)^d m}{a^2}+n\sqrt{\frac{\beta(1- \alpha)}{\beta-\alpha}} 2^{-\left(\log n+\log\frac{3}{\delta}+\frac{1}{2}\log\left(\frac{\beta(1-\alpha)}{\beta-\alpha}\right)\right)}   \nonumber\\
&\le\frac{\delta}{2}+\frac{\delta}{3}(1+o(1))\le \delta.\nonumber
 \end{align}
Inequality~(\ref{eq1}) is valid since the probability of failure over a design of size $m'$ is greater than the failure probability over a larger design of size $ m > m'$.  To prove that $c_\RsSD\le \frac{1}{(\ln 2)^2}$, it is easy to verify that
\begin{eqnarray*}
c_\RsSD\le \lim_{d\to\infty} \frac{1}{d(\ln 2)\left(H(\alpha)-\beta H\left(\frac{\alpha}{\beta}\right)\right)}&=& \frac{1}{(\ln2)x\log\frac{1}{1-e^{-x}}}
\end{eqnarray*} and then for $x=\ln 2$ the result follows.
\qed
\end{proof}

The lower bound for this model is given in the following lemma.
\begin{lemma}\label{kkkk} Let $M$ be an $m \times n$  RsSD   matrix with $ \alpha=s/m=\frac{x}{d} $ for any real number $0<x\le d$, and let $\beta$=$1-(1+\lambda)(1-\alpha)^d$, where $\lambda<1/10$ is any small constant. If
$$m=\frac{\ln n+\ln2+\frac{1}{2}\ln\left(\frac{\beta(1-\alpha)}{\beta-\alpha}\right)}{\left(H(\alpha)-\beta H\left(\frac{\alpha}{\beta}\right)\right)\ln 2},$$
then, with probability at least $3\lambda/16$, $M$ is not $(n,I)$-disjunct.
In particular
$$c^D_\RsSD(d)\ge \frac{1}{(\ln 2)\max_{0<x\le d}\left(H(\alpha)-\beta H\left(\frac{\alpha}{\beta}\right)\right)}   \mbox{\ \ \ \ and\ \ \ \ } c_\RsSD\ge \frac{1}{(\ln 2)^2}.$$
\end{lemma}

\begin{proof}
Let $X$ be the number of the good rows in $M$ and, let $A$ be the event that $M$ is not $(n,I)$-disjunct.
Then, by Markov inequality, Lemma~\ref{Markov}, we have,
\begin{eqnarray*}
n\frac{ {\beta m \choose \alpha m} }{{m \choose \alpha m}}=n\sqrt{\frac{\beta(1-\alpha)}{\beta-\alpha}} 2^{(\beta H(\frac{\alpha}{\beta})-H(\alpha))m}  (1+o(1))=\frac{1}{2}(1+o(1)),
\end{eqnarray*} and,
\begin{align*}
\Pr[A] &\ge \Pr[A\wedge X<(1+\lambda)\E[X]] \\
&= \Pr[A | X<(1+\lambda)(1-\alpha)^dm]\Pr[X<(1+\lambda)\E[X]] \\
&= \Pr[A | X<(1-\beta)m]\Pr[X<(1+\lambda)\E[X]] \\
&\ge \Pr[A | X=(1-\beta)m]\Pr[X<(1+\lambda)\E[X]] \\
&= \left(1-\left(1-\frac{ {\beta m \choose \alpha m} }{{m \choose \alpha m}}\right)^n\right) (1-\Pr[X\ge(1+\lambda)\E[X]]) \\
&\ge n\frac{ {\beta m \choose \alpha m} }{{m \choose \alpha m}}\left(1-\frac{n}{2}\frac{ {\beta m \choose \alpha m} }{{m \choose \alpha m}}\right)\left(1-\frac{1}{1+\lambda}\right)\\
&= \frac{3}{8}\left(1-\frac{1}{1+\lambda}\right)(1+o(1))\ge \frac{3}{16}\lambda.
\end{align*}
For any small constant $\lambda$,
\begin{eqnarray*}
\lim_{d\to\infty} \frac{1}{d(\ln 2)\left(H(\alpha)-\beta H\left(\frac{\alpha}{\beta}\right)\right)}&=& \frac{1}{(\ln2)x\log\frac{1}{1-(1+\lambda)e^{-x}}}.
\end{eqnarray*}
The value of $x\log({1}/({1-e^{-x}}))$ is minimal when $x=\ln 2$ (define $y=1-e^{-x}$ and show that the optimal point is when $y=1/2$), and then, $c_\RsSD\ge \frac{1}{(\ln 2)^2}$.
\qed

\end{proof}

\section{Random Uniform Transversal Design Model}
A design matrix $M$ is called \emph{transversal} if the rows of $M$ can be divided into disjoint families, where each family is a partition of all items  \cite{DH06\ignore{, DHWZ06, CD08}}. A well known method for constructing transversal designs is using a \emph{$q$-ary matrix}. A $q$-ary matrix $M'$ is a matrix over the alphabet $\Sigma=\{1,\ldots,q\}$, for some fixed $2\leq q\in[n]$. Transforming a $q$-ary matrix $M'$ to a binary matrix $M$ is as follows.  Each row  $r$ in $M'$ is translated to $q$ binary rows in $M$. For each $\sigma \in \Sigma$, replace each entry that is equal to $\sigma$ by $1$ and the others convert to $0$. Therefore, if the matrix $M'$ is of dimension $m'\times n$, then $M$ is an $m\times n$ binary matrix where $m = qm'$.
We say that a $q-$ary matrix $M'$ (UTDq algorithms) is a \emph{uniform random} matrix if its entries are chosen randomly and  independently to be any symbol of the alphabet with probability $1/q$. We say that $M$ is a \emph{$q$- transversal random matrix}, if there is a uniform random  $q-$ary matrix $M'$ such that $M$ is derived from $M'$ according to the previous procedure.

For ease of the analysis, we assume that $d\le q$. In the full paper, we show that all the results are also true for any $q\ge 1$. For $d\le q$, we define $$P_{q,d}=\left(\prod_{i=1}^d\left(\frac{i}{q}\right)^{R_{q,d,i}}\right)^{1/q^d},$$ where $R_{q,d,i}$ is the number of strings in $[q]^d$ that contains exactly $i$ symbols. It is easy to see that
$$R_{q,d,i}={q\choose i}N_{d,i},$$ where $N_{d,i}$ is the number of strings of length $d$ over the alphabet $\Sigma_i:=\{1,2,\ldots,i\}$ that contains all the symbols in $\Sigma_i$.

In this section we prove
\begin{theorem} We have
$$c^D_\UTDq(d)= \min_q\frac{q}{-\ln P_{q,d}}   \mbox{\ \ \ \ and\ \ \ \ } c_\UTDq=\frac{1}{(\ln 2)^2}.$$
\end{theorem}

For any set $K\subseteq [n]$, let $S_{i,K}(M)= \{ M_{i,t} | t\in K \}$ be the set of the symbols that appear in the entries that correspond to the row $i$ and the columns of $K$. Throughout this section, we will assume, w.l.o.g., that the set of the defective item is $I=[d]$. We first show
\begin{lemma}\label{hw} Let $M$ be an $m \times n$ $q$-transversal random matrix. The probability that $M$ is not $(n,[d])$-disjunct is
$$1-\left(1-\prod_{i=1}^{m'}\frac{|S_{i,[d]}(M')|}{q}\right)^{n-d}.$$
\end{lemma}
\begin{proof} The matrix $M$ is $(n,[d])$-disjunct if for every column $j>d$ in $M$, there is a row $i$ such that $M_{i,1}=\cdots=M_{i,d}=0$ and $M_{i,j}\not=0$. By the construction of $M$, this is equivalent to: for every column $j>d$ in $M'$, there is a row $i$ such that $M'_{i,j}\not\in\{M'_{i,1},\ldots,M'_{i,d}\}$. The probability that $M'_{i,j}\not\in\{M'_{i,1},\ldots,M'_{i,d}\}$ is $1-|S_{i,[d]}(M')|/{q}$, therefore, the result follows.\qed

\end{proof}

The following lemma provides an upper bound on $c^D_\TRNS(d)$.
\begin{lemma} Let $M$ be an $m \times n$ $q$-transversal random matrix where
\begin{equation}
m =  \frac{1}{(1-\lambda)}\frac{q\ln(2n/\delta)}{-\ln P_{q,d}}=\frac{q\ln(2n/\delta)}{-\ln P_{q,d}}+o(\ln n),
\end{equation} and $\lambda=\sqrt{(2q^{d+1}/m)\ln (2q^d/\delta)}$.
Then, with probability at least $1-\delta$, $M$ is $(n,[d])-$disjunct. Therefore,
$$c^D_\TRNS(d)\le \min_q\frac{q}{-\ln P_{q,d}}.$$
\end{lemma}
\begin{proof} Consider $M'$ of size $m'\times n$, where $m'=m/q$. For $v\in [q]^d$, let $W_v$ be the number of rows $i$ in $M'$ such that $(M'_{i,1},\ldots,M'_{i,d})=v$. Since the probability that $(M'_{i,1},\ldots,M'_{i,d})=v$ is $1/q^d$, by Chernoff bound, Lemma~\ref{Chernoff}, and the union bound we can say that,
$$
\Pr[(\exists v\in [q]^d)W_v\le (1-\lambda)m'/q^d]\le q^d e^{-\lambda^2 m'/2q^d}= \frac{\delta}{2},
$$
for $\lambda=\sqrt{(2q^d/m')\ln (2q^d/\delta)}$.
By Lemma~\ref{hw} and assuming that $W_v> (1-\lambda)m'/q^d$ for every $v$, the probability that the matrix $M$ is not $(n,[d])$-disjunct is
\begin{eqnarray*}
1-\left(1-\prod_{i=1}^{m'}\frac{|S_{i,[d]}(M')|}{q}\right)^{n-d}&\le& n \prod_{i=1}^{m'}\frac{|S_{i,[d]}(M')|}{q}= n\prod_{v\in [q]^d}\left(\frac{|\{v_1,\ldots,v_d\}|}{q}\right)^{W_v}\\
&=& n\prod_{i=1}^d\left(\frac{i}{q}\right)^{R_{q,d,i}W_v} \le nP_{q,d}^{(1-\lambda)m/q}\le  \frac{\delta}{2}.
\end{eqnarray*}
Therefore, by the union bound, the failure probability is at most $\delta$.\qed
\end{proof}

We now prove a lower bound on $c^D_\TRNS(d)$.
\begin{lemma} Let $M$ be an $m \times n$ $q$-transversal random matrix. Let
\begin{equation}
m =  \frac{1}{(1+\lambda)}\frac{q\ln(8(n-d))}{-\ln P_{q,d}}=\frac{q\ln n}{-\ln P_{q,d}}+o(\ln n)
\end{equation} where $\lambda=\sqrt{(3q^{d+1}/m)\ln (16q^d)}$.
Then, with probability at least $3/4$, $M$ is not $(n,[d])-$disjunct. Therefore,
$$c^D_\TRNS(d)\ge \min_q\frac{q}{-\ln P_{q,d}}.$$
\end{lemma}
\begin{proof} Consider $M'$ of size $m'\times n$ where $m'=m/q$. For $v\in [q]^d$, let $W_v$ be the number of rows $i$ in $M'$ such that $(M'_{i,1},\ldots,M'_{i,d})=v$. Since the probability that $(M'_{i,1},\ldots,M'_{i,d})=v$ is $1/q^d$, by Chernoff bound, Lemma~\ref{Chernoff}, and the union bound we have,
$$
\Pr[(\exists v\in [q]^d)W_v\ge (1+\lambda)m'/q^d]\le q^d e^{-\lambda^2 m'/(3q^d)}\le \frac{1}{8},
$$
for $\lambda=\sqrt{(3q^d/m')\ln (8q^d)}$.
By Lemma~\ref{hw} and assuming that $W_v< (1+\lambda)m'/q^d$ for every $v$, the probability that the matrix $M$ is not $(n,[d])$-disjunct is
\begin{eqnarray*}
1-\left(1-\prod_{i=1}^{m'}\frac{S_{i,[d]}(M')}{q}\right)^{n-d}&\ge& 1-exp\left((n-d)\prod_{i=1}^{m'}\frac{S_{i,[d]}(M')}{q}\right)\\
&\ge& 1-exp\left((n-d)P_{q,d}^{(1+\lambda)m/q}\right)\ge  \frac{7}{8}.
\end{eqnarray*}
Therefore, by the union bound, the failure probability is at least $3/4$.\qed
\end{proof}

It is not clear, however, how to compute the $c^D_\TRNS(d)/d$ when $d\to\infty$ in order to get $c_\TRNS$. The following two lemmas give a different analysis that enables us to approximate $c^D_\TRNS(d)$, and then to compute  $c_\TRNS$.

\begin{lemma}
\label{TRNS_UB}
Let $M$ be an $m\times n$ $q$-transversal random matrix where
\begin{equation}
\label{TRNS_UP_EQ}
m =  \frac{q\ln (n/\delta)}{-\ln\left(1-\left(1-\frac{1}{q}\right)^d\right)}.
\end{equation}
Then, with probability at least $1-\delta$, $M$ is an $(n,[d])$-disjunct matrix. In particular,
$$c^D_\TRNS(d)\le  \min_{q}\frac{q}{-\ln\left(1-\left(1-\frac{1}{q}\right)^d\right)} \mbox{\ \ \ \ and\ \ \ \ } c_\TRNS\le \frac{1}{(\ln2)^2} .$$
\end{lemma}
\begin{proof}
Consider $M'$ of size $m'\times n$ where $m'=m/q$. The probability that $M$ is not $(n,[d])$-disjunct is the probability that there is a column $j\in\{d+1, \ldots, n\}$ such that for each row $i$ in $M'$, $M'_{i,j} \in S_{i,[d]}(M')$. Let $A$ denote the event that the matrix $M$ is not $(n,[d])$-disjunct. Then,

\begin{eqnarray*}
\Pr[A]& \leq & n\left (1- \left (  1-\frac{1}{q} \right )^d\right )^{m'} = n\left (1- \left (  1-\frac{1}{q} \right )^d\right )^{m/q}= \delta. \\
\end{eqnarray*}
The probability here is calculated by first choosing $M'_{i,j}$ and then choose $M'_{i,\ell}$, $\ell=1,\ldots,d$. To evaluate $c^D_\TRNS(d)$, let $x=q/d$. Then, when $d\to \infty$, we have,
$$\frac{c^D_\TRNS(d)}{d}\le \min_{x>1}\frac{x}{-\ln\left(1-\left(1-\frac{1}{xd}\right)^d\right)}\to
\min_{x>1} \frac{x}{-\ln(1-e^{-1/x})}=\frac{1}{(\ln2)^2}.\qed$$
\end{proof}

%%%%%%%%%%%%%%%%%%%%%%%%%%%%%%
In the following Lemma~\ref{TRNS_LP_EQ}, we prove a tight lower bound for $c_\TRNS$.

\begin{lemma}\label{TRNS_LP_EQ} Let  be an $m\times n$ $q$-transversal random matrix where
$$m=\frac{q\ln((n-d)/2)}{-\ln\left({1-\left(1-\frac{1}{q}\right)^d-O(q^{-1/3})}\right)}.$$
Then, with probability at least $1/4$, the matrix $M$ is not $(n,[d])$-disjunct.

In particular,
we have
$$c_\TRNS\ge \frac{1}{\ln^2(2)} .$$
\end{lemma}

\begin{proof} First, we prove the result when $q>10d\ln d$. Since $\sum_{i=1}^d R_{q,d,i}=q^d$,
$$P_{q,d}:=\left(\prod_{i=1}^d\left(\frac{i}{q}\right)^{R_{q,d,i}}\right)^{1/q^d}\ge \left(\prod_{i=1}^d\left(\frac{1}{q}\right)^{R_{q,d,i}}\right)^{1/q^d}=\frac{1}{q}.$$
Therefore, the constant that we get in this case is
$\frac{q}{-d\ln P_{q,d}}\ge 5\ge \frac{1}{(\ln 2)^2}.$ Thus, we may assume that $q\le 10d\ln d$.

As in the proof of Lemma~\ref{TRNS_UB}, the probability that $M$ is not $(n,[d])$-disjunct is the probability that there is a column $j\in\{d+1,\ldots, n\}$ such that, for each row $i$ in $M'$, $M'_{i,j} \in S_{i,[d]}(M')$. Let $Y_{i,v}$, $i=1,\ldots, m'$, $v=1,\ldots,q$ be a random variable that is equal to $1$ if $v\not\in S_{i,[d]}(M')$ and $0$ otherwise. Let $Y_i=q-|S_{i,[d]}(M')|=Y_{i,1}+\cdots+Y_{i,q}$. Let $\mu:=\E[Y_i]=q(1-1/q)^b$ and $\mu_2:=\E[Y_{i,v_1}Y_{i,v_2}]=(1-2/q)^b$. Then, by Chebychev bound, Lemma~\ref{Chebychev},
\begin{eqnarray*}
\Pr[|Y_i-q(1-1/q)^d|\ge q^{2/3}]&\le & \frac{q\left(1-\frac{1}{q}\right)^b+q(q-1)\left(1-\frac{2}{q}\right)^b-q^2\left(1-\frac{1}{q}\right)^{2b}}{q^{4/3}}\\
&\le& q^{-1/3}\left(1-\frac{1}{q}\right)^d.
\end{eqnarray*}
By Markov bound, Lemma~\ref{Markov}, with probability at least $3/4$, more than $m'-4(1-{1}/{q})^dm'/q^{1/3}$ rows $i\in [m']$ in $M'$ satisfy the property: $$|S_{i,[d]}(M')|=q-Y_i\ge q-q(1-1/q)^d-q^{2/3}.$$

Since the number of strings in $[q]^d$ with at most $d/4$ symbols is at most ${q\choose d/4}\left(\frac{d}{4}\right)^d$, we can conclude that,
\begin{eqnarray*}
\Pr[|S_{i,[d]}(M')|\le d/4]&\le& {q\choose d/4}\left(\frac{d/4}{q}\right)^d\le \left(\frac{eq}{d/4}\right)^{d/4}\left(\frac{d/4}{q}\right)^d\\
&=& e^{d/4}\left(\frac{d/4}{q}\right)^{3d/4}\le e^{d/4}\left(\frac{1}{4}\right)^{3d/4}\le 2^{-d}.
\end{eqnarray*}
By Markov bound, with probability at least $3/4$, less than $2^{-d+2}m'$ rows in $M'$ satisfy $|S_{i,[d]}(M')|=q-Y_i\le d/4$.

In addition, the probability that the matrix $M$ is not $(n,[d])$-disjunct is,
\begin{eqnarray*}
1-\left(1-\prod_{i=1}^{m'}\frac{S_{i,[d]}(M')}{q}\right)^{n-d}.
\end{eqnarray*}
On the other hand, with probability at least $1/2$, for $d\le q\le 10d\ln d$, we can say that,
\begin{eqnarray*}
\prod_{i=1}^{m'}\frac{S_{i,[d]}(M')}{q}&\ge & \left(1-\left(1-\frac{1}{q}\right)^d-q^{-1/3}\right)^{m'}\left(\frac{1}{q}\right)^{2^{-d+2}m'}
\left(\frac{d}{4q}\right)^{4(1-{1}/{q})^dm'/q^{1/3}}\\
&\ge & \left(1-\left(1-\frac{1}{q}\right)^d-O(q^{-1/3})\right)^{m'}\ge \frac{2}{n-d}.
\end{eqnarray*}
Hence we get,
\begin{eqnarray*}
1-\left(1-\prod_{i=1}^{m'} \frac{S_{i,[d]}(M')}{q}\right)^{n-d}\ge 1-e^{-2}\ge 3/4.
\end{eqnarray*}
Therefore, with probability at least $1/4$, $M$ is not $(n,[d])$-disjunct.
Let $x=q/d\le 10\ln d$. When $d\to \infty$, we have,
$$\frac{c^D_\TRNS(d)}{d}\ge \min_{x>1}\frac{x}{-\ln\left(1-\left(1-\frac{1}{xd}+O(\frac{1}{d^3})\right)^d\right)}\to
\min_{x>1} \frac{x}{-\ln(1-e^{-1/x})}=\frac{1}{(\ln2)^2}.\qed$$\
\end{proof}

\bibliographystyle{plain}
\bibliography{Ref}

\begin{thebibliography}{10}

\bibitem{DR82}
V.~V.~Rykov A.~G.~D'yachkov.
\newblock Bounds on the length of disjunctive codes.
\newblock {\em Problemy Peredachi Informatsii.}, 18:7--13, 1982.

\bibitem{Ang88}
Dana Angluin.
\newblock Queries and concept learning.
\newblock {\em Machine Learning}, 2(4):319--342, 1987.

\bibitem{BBKT95}
D.~J. Balding, W.~J. Bruno, D.C Torney, and E.~Knill.
\newblock A comparative survey of non-adaptive pooling designs.
\newblock In Terry Speed and Michael~S. Waterman, editors, {\em Genetic Mapping
  and DNA Sequencing}, pages 133--154, New York, NY, 1996. Springer New York.

\bibitem{BKB95}
W.J. Bruno, E.~Knill, D.J. Balding, D.C. Bruce, N.A. Doggett, W.W. Sawhill,
  R.L. Stallings, C.C. Whittaker, and D.C. Torney.
\newblock Efficient pooling designs for library screening.
\newblock {\em Genomics}, 26(1):21 -- 30, 1995.

\bibitem{BBHHKS18}
Nader~H. Bshouty, V.~E. Bshouty-Hurani, G.~Haddad, T.~Hashem, F.~Khoury, and
  O.~Sharafy.
\newblock Elementary proofs of some stirling bounds.
\newblock {\em arXiv:1802.07046}, 2018.

\bibitem{BDKS16}
Nader~H. Bshouty, Nuha Diab, Shada~R. Kawar, and Robert~J. Shahla.
\newblock Non-adaptive randomized algorithm for group testing.
\newblock In {\em International Conference on Algorithmic Learning Theory,
  {ALT} 2017, 15-17 October 2017, Kyoto University, Kyoto, Japan}, pages
  109--128, 2017.

\bibitem{Ci13}
Ferdinando Cicalese.
\newblock {\em Group Testing}, pages 139--173.
\newblock Springer Berlin Heidelberg, Berlin, Heidelberg, 2013.

\bibitem{CM03}
G.~Cormode and S.~Muthukrishnan.
\newblock What's hot and what's not: Tracking most frequent items dynamically.
\newblock {\em ACM Trans. Database Syst.}, 30(1):249--278, March 2005.

\bibitem{DA12}
Peter Damaschke and Azam~Sheikh Muhammad.
\newblock Randomized group testing both query-optimal and minimal adaptive.
\newblock In {\em SOFSEM 2012: Theory and Practice of Computer Science}, pages
  214--225, Berlin, Heidelberg, 2012. Springer Berlin Heidelberg.

\bibitem{D43}
Robert Dorfman.
\newblock The detection of defective members of large populations.
\newblock {\em The Annals of Mathematical Statistics}, 14(4):436--440, 1943.

\bibitem{DH00}
Ding-Zhu Du and Frank~K. Hwang.
\newblock {\em Combinatorial Group Testing and Its Applications}.
\newblock World Scientfic Publishing, 1993.

\bibitem{DH06}
Ding-Zhu Du and Frank~K. Hwang.
\newblock {\em Pooling Designs And Nonadaptive Group Testing: Important Tools
  For Dna Sequencing}.
\newblock World Scientfic Publishing, 2006.

\bibitem{ER63}
Paul Erd{\"o}s and Alfr{\'e}d R{\'e}nyi.
\newblock On two problems of information theory.
\newblock pages 241--254. Publications of the Mathematical Institute of the
  Hungarian Academy of Sciences, 1963.

\bibitem{F96}
Zoltán Füredi.
\newblock On $r$-cover-free families.
\newblock {\em Journal of Combinatorial Theory, Series A}, 73(1):172 -- 173,
  1996.

\bibitem{HL02}
E.~S. {Hong} and R.~E. {Ladner}.
\newblock Group testing for image compression.
\newblock {\em IEEE Transactions on Image Processing}, 11(8):901--911, Aug
  2002.

\bibitem{H72}
F.~K. Hwang.
\newblock A method for detecting all defective members in a population by group
  testing.
\newblock {\em Journal of the American Statistical Association},
  67(339):605--608, 1972.

\bibitem{H00}
F.~K. Hwang.
\newblock Random k-set pool designs with distinct columns.
\newblock {\em Probability in the Engineering and Informational Sciences},
  14(1):49–56, 2000.

\bibitem{HL01}
F.~K. Hwang and Y.~C. Liu.
\newblock The expected numbers of unresolved positive clones for various random
  pool designs.
\newblock {\em Probability in the Engineering and Informational Sciences},
  15(1):57–68, 2001.

\bibitem{HL04}
F.~K. Hwang and Y.~C. Liu.
\newblock A general approach to compute the probabilities of unresolved clones
  in random pooling designs.
\newblock {\em Probability in the Engineering and Informational Sciences},
  18(2):161–183, 2004.

\bibitem{HL03}
Frank~K. Hwang and Youwu Liu.
\newblock Random pooling designs under various structures.
\newblock {\em Journal of Combinatorial Optimization}, 7:339--352, 2003.

\bibitem{KS64}
W.~{Kautz} and R.~{Singleton}.
\newblock Nonrandom binary superimposed codes.
\newblock {\em IEEE Transactions on Information Theory}, 10(4):363--377,
  October 1964.

\bibitem{MP04}
Anthony~J. Macula and Leonard~J. Popyack.
\newblock A group testing method for finding patterns in data.
\newblock {\em Discrete Applied Mathematics}, 144(1):149 -- 157, 2004.
\newblock Discrete Mathematics and Data Mining.

\bibitem{PR11}
Ely Porat and Amir Rothschild.
\newblock Explicit nonadaptive combinatorial group testing schemes.
\newblock {\em {IEEE} Trans. Information Theory}, 57(12):7982--7989, 2011.

\bibitem{ND00}
Hung Q~Ngo and Ding-Zhu Du.
\newblock A survey on combinatorial group testing algorithms with applications
  to dna library screening.
\newblock {\em DIMACS Series in Discrete Mathematics and Theoretical Computer
  Science}, 55, 12 2000.

\bibitem{R94}
Miklós Ruszinkó.
\newblock On the upper bound of the size of the r-cover-free families.
\newblock {\em Journal of Combinatorial Theory, Series A}, 66(2):302 -- 310,
  1994.

\bibitem{S85}
András Seb{\"{o}}.
\newblock On two random search problems.
\newblock {\em Journal of Statistical Planning and Inference}, 11(1):23 -- 31,
  1985.

\bibitem{W85}
J.~{Wolf}.
\newblock Born again group testing: Multiaccess communications.
\newblock {\em IEEE Transactions on Information Theory}, 31(2):185--191, March
  1985.

\end{thebibliography}

%\bibliographystyle{plain}
%\bibliography{Ref}
%\bibliographystyle{Abbrv}

\appendix
\section{ Preliminaries}
In this section, we give some preliminary results that will be used throughout the paper.

\begin{lemma}\label{Markov}{\bf Markov bound}
Let $X$ be a non-negative random variable, and let $a > 1$, then $$ \Pr[X\ge a\E[X]]\le\frac{1}{a}. $$
\end{lemma}

\begin{lemma}\label{Chebychev}{\bf Chebychev inequality}
Let X be a random variable with a finite expected value $\mu=\E[X]$, and finite non-zero variance $\Var[X]$. Then, for any real number $a>0$, $$ Pr[|X-\mu|\ge a]\le \frac{\Var[X]}{a^2}=\frac{\E[X^2]-\mu^2}{a^2}.$$
In particular, let $X=X_1+\cdots+X_n$, where each $X_i$ is a random variable that takes values from $\{0,1\}$.
If $\E[X_iX_j]=\mu_2$ for all $1\le i<j\le n$ then,
$$ Pr[|X-\mu|\ge a]\le \frac{\mu+n(n-1)\mu_2-\mu^2}{a^2}.$$
\end{lemma}

\begin{lemma}\label{Chernoff}{\bf Chernoff bound}
Let $X_1,\ldots,X_n$ be n independent random variables that take values in $\{0,1\}$. Let $X=X_1+\cdots+X_n$ and $ \mu=\E[X]$. Then, for any $0\le \lambda \le 1$ we have,
$$\Pr\left[X\le (1-\lambda) \mu \right] \le e^{-\frac{\lambda ^2 \mu}{2}},$$ and,  $$\Pr\left[X\ge (1+\lambda) \mu \right] \le e^{-\frac{\lambda ^2 \mu}{3}}. $$
\end{lemma}

\begin{lemma}\label{mono} Let $A,B_1,\ldots,B_t$ be events such that $B_i\cap B_j=\emptyset$ for every $i\not=j$ and $\Pr[A|B_i]\ge \Pr[A|B_1]$ (resp. $\Pr[A|B_i]\leq \Pr[A|B_1]$) for every $i$. Then,
$\Pr[A|B_1\cup \cdots \cup B_t]\ge \Pr[A|B_1].$ (resp. $\Pr[A|B_1\cup \cdots \cup B_t]\leq \Pr[A|B_1].$)
\end{lemma}
\begin{proof} By conditional probability we get,
\begin{eqnarray*}
\Pr[A|\cup_iB_i]&=&\frac{\Pr[A\cap (\cup_i B_i) ]}{\Pr[\cup_i B_i]} = \frac{\sum_{i=1}^t\Pr[A\cap B_i]}{\sum_{i=1}^t \Pr[B_i]} \\
&=&\frac{\sum_{i=1}^t\Pr[A| B_i]\Pr[B_i]}{\sum_{i=1}^t \Pr[B_i]}\ge  \Pr[A|B_1].\qed
\end{eqnarray*}
\end{proof}
\begin{lemma} \label{STRLNG}
 Let $0< \alpha <1$ be any real number such that $\alpha n$ is an integer. Then
\begin{align*} { n \choose \alpha n} = \frac{1}{\sqrt{2\pi \alpha (1-\alpha)n}} 2^{H(\alpha)n} e^{\frac{\alpha(1-\alpha)-1}{12\alpha(1-\alpha)n}+
\Theta\left(\frac{1}{(\min(\alpha,1-\alpha)n)^3}\right)}
\end{align*}
where $H(x)=-x\log_2 x - (1-x) \log_2 (1-x).$
\end{lemma}
\begin{proof}
By Sterling bound~\cite{BBHHKS18},
\begin{align*} { n \choose \alpha n}  &= \frac{n!}{(\alpha n)! ((1-\alpha)n)!} \\ &=\frac{ \sqrt{2 \pi n} (\frac{n}{e}) ^n e^{ \frac{1}{12n}+\Theta\left(\frac{1}{n^3}\right) } }{\sqrt{2\pi \alpha n} \left(\frac{\alpha n}{e}\right)^{\alpha n} e^{\frac{1}{12\alpha n} +\Theta\left(\frac{1}{(\alpha n)^3}\right)} \sqrt{2\pi (1-\alpha) n} \left(\frac{(1-\alpha)n}{e}\right)^{(1-\alpha)n}e^{\frac{1}{12(1-\alpha)n}+
\Theta\left(\frac{1}{((1-\alpha)n)^3}\right)} } \\ &= \frac{1}{\sqrt{2\pi \alpha (1-\alpha)n}} \left(\frac{1}{\alpha^\alpha (1-\alpha)^{1-\alpha}}\right)^n e^{\frac{\alpha(1-\alpha)-1}{12\alpha(1-\alpha)n}+
\Theta\left(\frac{1}{(\min(\alpha,1-\alpha)n)^3}\right)} \\ &= \frac{1}{\sqrt{2\pi \alpha (1-\alpha)n}} 2^{H(\alpha)n} e^{\frac{\alpha(1-\alpha)-1}{12\alpha(1-\alpha)n}+\Theta\left(\frac{1}
{(\min(\alpha,1-\alpha)n)^3}\right)}.\qed
\end{align*}
%%%%%%%%%%%%%%%
\ignore{
Let $\beta=1-\alpha$. By Sterling bound~\cite{BBHHKS18},
\begin{align*} { n \choose \alpha n}  &= \frac{n!}{(\alpha n)! ((1-\alpha)n)!} \\ 
&=\frac{ \sqrt{2 \pi n} (\frac{n}{e}) ^n e^{ \frac{1}{12n}+\Theta\left(\frac{1}{n^3}\right) } }{\sqrt{2\pi \alpha n} \left(\frac{\alpha n}{e}\right)^{\alpha n} e^{\frac{1}{12\alpha n} +\Theta\left(\frac{1}{(\alpha n)^3}\right)} \sqrt{2\pi \beta n} \left(\frac{\beta n}{e}\right)^{\beta n}e^{\frac{1}{12\beta n}+
\Theta\left(\frac{1}{(\beta n)^3}\right)} } \\ &= \frac{1}{\sqrt{2\pi \alpha \beta n}} \left(\frac{1}{\alpha^\alpha \beta^{\beta}}\right)^n e^{\frac{\alpha \beta-1}{12\alpha \beta n}+
\Theta\left(\frac{1}{(\min(\alpha,\beta)n)^3}\right)} \\ &= \frac{1}{\sqrt{2\pi \alpha \beta n}} 2^{H(\alpha)n} e^{\frac{\alpha \beta-1}{12\alpha \beta n}+\Theta\left(\frac{1}
{(\min(\alpha,\beta)n)^3}\right)}.\qed
\end{align*}
}

\end{proof}

\begin{lemma}\label{ID} Let $B_1,\ldots,B_t$ be events and $A=B_1\vee B_2\vee \cdots \vee B_t$. Then
$$\Pr[A]\ge \sum_{i=1}^t \Pr[B_i]-\sum_{1<i<j<t} \Pr[B_i\wedge B_j].$$
\end{lemma}

\begin{lemma}\label{varv}
Let $M$ be an $m \times n $ RsSD matrix. Let  $\alpha=s/m$. Let $X_i$ be a random variable that is equal to $1$ if the $i$th row is good and is equal to $0$ otherwise.
Let $X=\sum_{i=1}^{m} X_i $ be the number of good rows in $M$. Then $$\Var[X]\le (1-\alpha)^d m.$$
\end{lemma}
\begin{proof} Since the columns of $M$ are chosen randomly and independently,  $\E[X_i]=(1-\alpha)^d$ and, therefore, $\E[X]=(1-\alpha)^dm$. Moreover, for $i\not=j$, we have
$$\E[X_iX_j]=\left(\frac{{m-2\choose s}}{{m\choose s}}\right)^d=(1-\alpha)^d\left(1-\frac{s}{m-1}\right)^d \le (1-\alpha)^{2d}.$$
Hence, we can conclude,
\begin{align*}
\Var[X] &= \E[X^2]-\E[X]^2 =  \E\left[ \sum_{i=1}^{m} {X_i} + 2 \sum_{1\leq i<j\le m}^{} {X_i X_j} \right] - (1-\alpha)^{2d} m^2 \\ &\le (1-\alpha)^d m +2 {m \choose 2} (1-\alpha)^{2d} - (1-\alpha)^{2d} m^2 \le (1-\alpha)^d m.\qed
\end{align*}
\end{proof}

\end{document}